\documentclass[journal,twoside]{IEEEtran}

\usepackage{cite}

\usepackage{amsmath,amssymb}
\usepackage{graphicx}
\usepackage{booktabs}
\usepackage{multirow}

\usepackage{algorithm}
\usepackage{algpseudocode}

\usepackage{verbatim}
\usepackage{url}

\usepackage{etoolbox}

\makeatletter

\def\@IEEEBIOskipN{2\baselineskip}

\expandafter\patchcmd\csname\string\IEEEbiography\endcsname
  {\vskip \@IEEEBIOskipN plus 1fil minus 0\baselineskip}
  {\vskip \@IEEEBIOskipN}
  {\typeout{IEEE biography spacing patch success}}
  {\typeout{IEEE biography spacing patch failed}}

\makeatother

\newtheorem{appproposition}{Proposition}

\begin{document}


\title{Physics-Driven Independent Pair Generation for Iterative
Self-Supervised Low-Dose CT Denoising}


\author{%
Xianlei~Han,
Shaoyu~Wang,
Jiancheng~Fang,
Weiwen~Wu,~\IEEEmembership{Member,~IEEE,}
and~Qiegen~Liu,~\IEEEmembership{Senior~Member,~IEEE}%
\thanks{This work was supported in part by the National Key Research and Development Program of China under Grant 2023YFF1204300 and Grant 2023YFF1204302. (Xianlei Han and Shaoyu Wang are co-first authors). (Corresponding author: Qiegen Liu).}%
\thanks{Xianlei Han is with the School of Mathematics and Computer
Sciences, Nanchang University, Nanchang 330031, China
(e-mail: 9110124183@email.ncu.edu.cn).}%
\thanks{Shaoyu Wang, Jiancheng Fang, and Qiegen Liu are with the
School of Information Engineering, Nanchang University,
Nanchang 330031, China
(e-mail: wangshaoyu@ncu.edu.cn;
d5283377@163.com;
liuqiegen@ncu.edu.cn).}%
\thanks{Weiwen Wu is with the School of Biomedical Engineering,
Sun Yat-sen University, Shenzhen 510275, Guangdong, China
(e-mail: wuweiw7@mail.sysu.edu.cn).}%
}

\markboth{IEEE Transactions on Circuits and Systems for Video Technology}%
{Han \MakeLowercase{\textit{et al.}}: Physics-Driven Independent Pair Generation and Iteration}

\maketitle

\begin{abstract}

Low-dose computed tomography (LDCT) measurements contain mixed Poisson–Gaussian noise. However, most self-supervised methods rely on generic image statistics and do not explicitly model this noise, which may limit their ability to effectively suppress realistic LDCT noise. To address this issue, we propose a physics-driven framework with cross-domain iteration for self-supervised LDCT denoising. The proposed framework proceeds in three main steps. First, a learned sinogram prior and the LDCT noise model guide posterior inference of photon counts, enabling separation of the Poisson and Gaussian components. Second, the separated Poisson and Gaussian components are respectively processed by binomial thinning and Gaussian data thinning to construct two branches, and residual scaling matches each branch’s noise level to that of the observation, yielding a training pair with approximately independent noise realizations from one low-dose measurement. Finally, the pair is used to train an image-domain network whose forward-projected outputs update the prior. Through cross-domain iteration, the prior and the training pair are progressively refined while maintaining consistency with CT acquisition physics. Experiments on simulated data from AAPM, LIDC-IDRI, and LoDoPaB-CT and on real LDCT data show consistent gains over the evaluated self-supervised baselines across dose levels, with performance comparable to the evaluated supervised baseline.

\end{abstract}

\begin{IEEEkeywords}
Low-dose CT, self-supervised denoising, cross-domain iteration, independent pairs, Poisson--Gaussian noise.
\end{IEEEkeywords}

\section{Introduction}
\label{sec:introduction}

\IEEEPARstart{L}{ow-dose} computed tomography (LDCT) reduces radiation exposure, but increased mixed Poisson--Gaussian noise degrades image quality and diagnostic reliability~\cite{ma2012variance,xie2017robust}. Traditional methods, including BM3D~\cite{dabov2007bm3d}, dictionary learning~\cite{elad2006image}, and total variation~\cite{rudin1992nonlinear}, depend on handcrafted priors and can be computationally demanding. Supervised methods such as RED-CNN~\cite{chen2017redcnn}, CoreDiff~\cite{gao2024corediff}, and PFCM~\cite{hein2025pfcm} provide substantially better denoising, but require paired low-dose and normal-dose data that are difficult to acquire clinically. 

\begin{figure}[!t]
    \centering
    \includegraphics[width=\columnwidth]{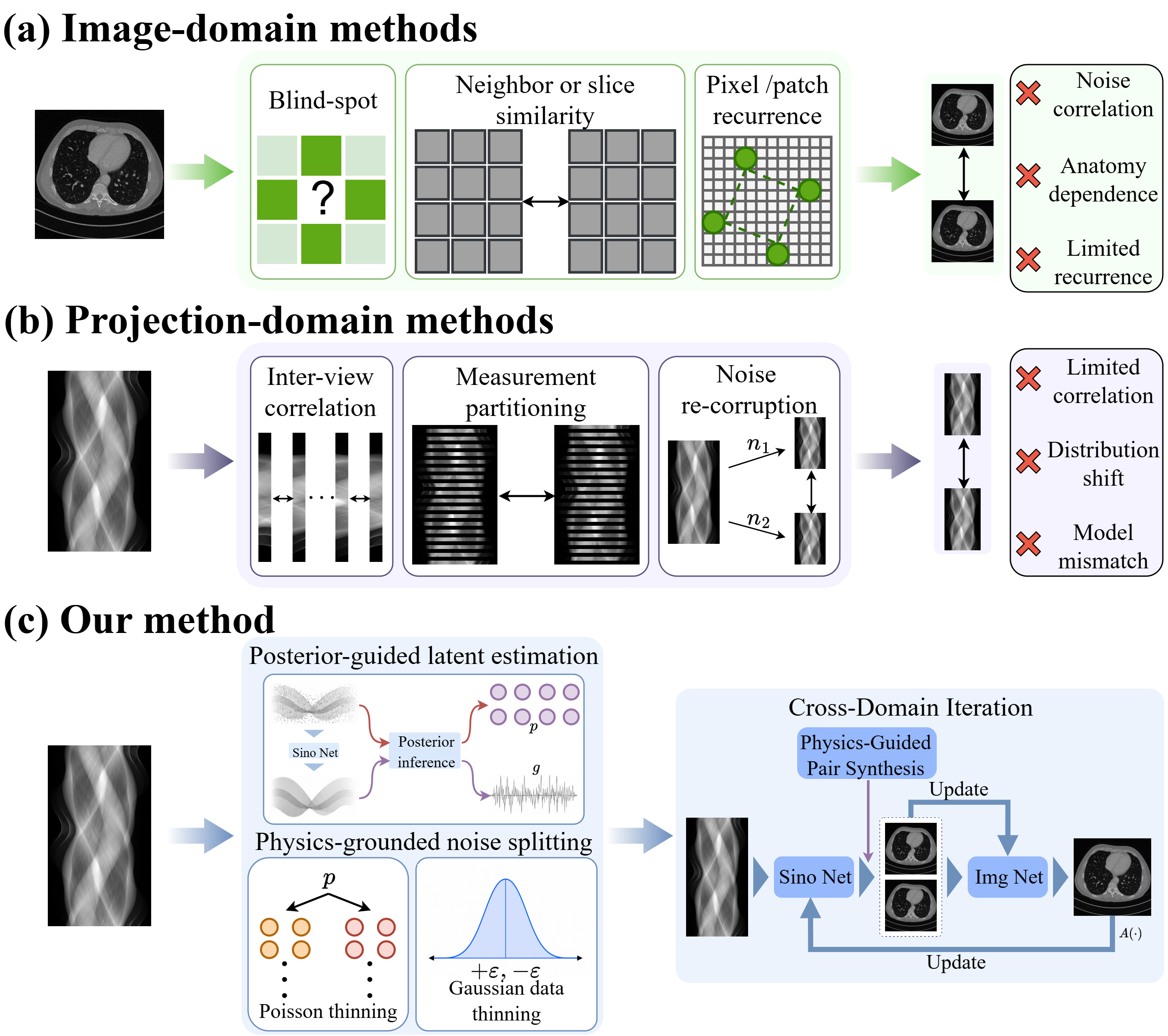}
    \caption{Representative self-supervised methods for LDCT denoising: (a) image-domain methods, (b) projection-domain methods, and (c) the proposed framework.}
    \label{fig:self_supervised_strategies}
\end{figure}

Noise2Noise (N2N)~\cite{lehtinen2018noise2noise} showed that paired noisy observations can replace clean targets under suitable statistical conditions. Existing self-supervised LDCT methods generally operate in the image or projection domain. Image-domain methods use blind-spot prediction, neighborhood subsampling, internal similarity, noise modeling, or aligned pairs~\cite{krull2019noise2void,huang2021neighbor2neighbor,wang2022blind2unblind,niu2023noise2sim,zhao2024wiald2nd,wang2025svb,liu2025sdcnn,wang2026nabsn}. Projection-domain methods use raw measurements through projection splitting, projection denoising, or photon-count splitting~\cite{yuan2020half2half,wu2023usgf,choi2023projection,an2024sdbdnet,wang2025emulating,lu2026left2right}. Fig.~\ref{fig:self_supervised_strategies} summarizes these representative strategies together with the proposed method. 

To satisfy the N2N assumptions, the two branches of a training pair must contain independent noise realizations. Existing methods rarely integrate mixed LDCT noise, acquisition physics, and cross-domain information, limiting their ability to meet these requirements. In the projection domain, photon-count fluctuations and electronic readout noise follow different distributions. Reconstruction converts their mixture into spatially structured image noise, so generic image-domain resampling may not preserve the acquisition statistics. Projection partitioning may also alter the effective dose, causing a noise-statistics mismatch between training and inference. 

To address these issues, we propose a physics-driven iterative framework that constructs training pairs with approximately independent noise. Specifically, a learned sinogram prior supports posterior inference of latent photon counts, enabling separation of the mixed noise components. The separated components are then processed by binomial and Gaussian data thinning, respectively, to form two stochastic branches. Residual scaling further matches each branch’s noise level to that of the observation. Finally, cross-domain iteration refines the prior with image-domain feedback and regenerates more reliable pairs, linking projection-domain pair construction with image-domain restoration.

The main contributions are summarized as follows:
\begin{itemize}
\item A training-pair construction method for mixed Poisson--Gaussian observations is proposed. By inferring latent photon counts and disentangling the noise components, approximately independent training pairs are generated from a single low-dose observation.

\item Cross-domain iteration is introduced to refine the sinogram prior with image-domain feedback. It thus improves pair reliability, and enhances physical consistency between the projection and image domains.

\item The independence of the constructed training pairs and the variance-matching rule are theoretically established under ideal conditions of the Poisson--Gaussian noise model. Extensive experiments validate the effectiveness of the proposed framework across different dose levels and datasets.
\end{itemize}

\section{Related Work}
\label{sec:related_work}

\subsection{Image-Domain Methods}
\label{subsec:image}

Image-domain methods mainly exploit statistics and redundancy in noisy observations. Noise2Void~\cite{krull2019noise2void}, Noise2Self~\cite{batson2019noise2self}, and Blind2Unblind~\cite{wang2022blind2unblind} introduce and advance blind-spot denoising. Neighbor2Neighbor~\cite{huang2021neighbor2neighbor} constructs pairs from neighboring pixels. Recent methods address correlated noise through tunable blind spots or re-noising consistency, including AT-BSN~\cite{chen2024atbsn} and Positive2Negative~\cite{li2025positive2negative}. For LDCT, Noise2Sim~\cite{niu2023noise2sim} uses nonlocal similarity, WIA-LD2ND~\cite{zhao2024wiald2nd} aligns wavelet-domain images, and SVB~\cite{wang2025svb} combines similar pairs with visual blind spots. SDCNN~\cite{liu2025sdcnn} separates signal and two noise components, whereas NA-BSN~\cite{wang2026nabsn} constrains spatially correlated CT noise. However, these methods do not explicitly model CT acquisition physics or noise formation, and thus cannot adequately characterize realistic CT noise. Moreover, without directly using projection measurements or the forward model, their consistency with the CT acquisition process is limited.

\subsection{Projection-Domain Methods}
\label{subsec:measurement_ssl}

Projection-domain methods use CT measurements more directly. Half2Half partitions projections from one acquisition~\cite{yuan2020half2half}. USGF~\cite{wu2023usgf} exploits projection structure, Choi \emph{et al.}~\cite{choi2023projection} denoise projections, and SDBDNet~\cite{an2024sdbdnet} enables dual-domain collaboration. OSDM~\cite{huang2024osdm} introduces one-sample unsupervised diffusion learning, while the low-rank angular prior guided multi-diffusion model extends this strategy to few-shot settings~\cite{zhang2024rapmd}. Wang \emph{et al.}~\cite{wang2025emulating} use binomial thinning of Poisson counts to generate pairs. Left2Right~\cite{lu2026left2right} constructs dual-domain self-supervised LDCT training data. These methods connect self-supervised learning more closely to CT acquisition. However, many still rely on specific partitions, measurement conditions, or simplified noise assumptions, so generated samples may depend on the acquisition partition.

To our knowledge, no existing approach jointly models mixed Poisson--Gaussian noise formation and enforces imaging constraints across the projection and image domains. Both elements are important because pair statistics originate in projection measurements, whereas denoising quality is evaluated after reconstruction. Accordingly, our framework uses the CT noise model for pair construction and cross-domain feedback for acquisition consistency.

\section{Methodology}

The proposed methodology comprises five components: Noise modeling and training-pair motivation (Sec.~\ref{subsec:motivation}), sinogram prior learning (Sec.~\ref{subsec:sino_prior}), posterior-guided pair construction (Sec.~\ref{sec:posterior_compensation}), cross-domain iterative training (Sec.~\ref{subsec:dual_domain}), inference and reconstruction (Sec.~\ref{subsec:inference}).

\subsection{Motivation}
\label{subsec:motivation}

Pair construction must represent mixed projection noise while remaining consistent with the CT imaging model. We describe each LDCT projection measurement as
\begin{equation}
Y=C+\varepsilon,
\label{eq:poisson_gaussian_observation}
\end{equation}
where $Y$ is the low-dose projection measurement, $C\sim\operatorname{Poisson}(\mu)$ is the photon count with mean $\mu$, and $\varepsilon\sim\mathcal{N}(0,\sigma^2)$ denotes Gaussian readout noise with variance $\sigma^2$.

Photon-count fluctuations depend on transmitted intensity, whereas electronic readout noise follows a separate Gaussian process. We therefore construct projection-domain pairs from both sources and their statistics. The CT imaging model provides a complementary physical constraint. Because sinogram-prior errors can propagate to the generated samples, cross-domain iteration forward projects restored images to update the prior for the next pair-construction round. This feedback reduces prior error and preserves acquisition consistency.

\subsection{Sinogram-Domain Prior Learning}
\label{subsec:sino_prior}

\begin{figure*}[htbp]
    \centering
    \includegraphics[width=\textwidth]{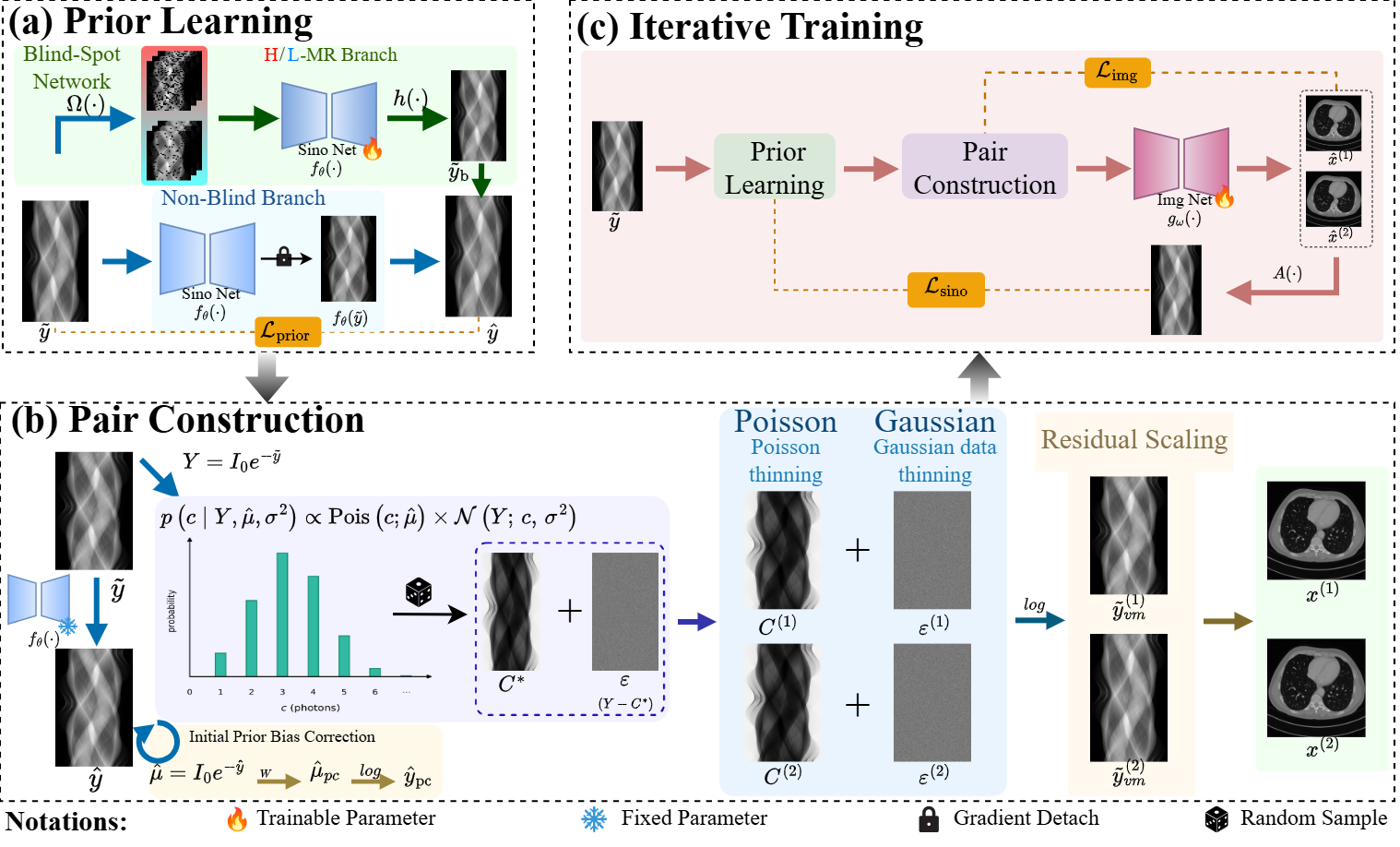}
    \caption{Proposed framework. (a) Sinogram prior learning with high and low masking ratios and a non-blind branch. (b) Posterior-guided Poisson--Gaussian pair construction with residual scaling for variance matching. (c) Cross-domain iterative training.}
    \label{fig:framework}
\end{figure*}

Fig.~\ref{fig:framework}(a) illustrates the learning process of the sinogram prior. Given a noisy sinogram $\tilde{y}$, $\Omega(\cdot)$ partitions its pixels into $K$ mutually exclusive masks, $\{M_m\}_{m=1}^{K}$. For each mask, the masked locations are filled with weighted estimates from neighboring pixels, and the resulting masked views are concatenated and processed by $f_{\theta}$. The mapping operator $h(\cdot)$ then restores the blind-spot predictions to their original coordinates: $\hat{y}_{\mathrm{b}}=h(f_{\theta}(\Omega(\tilde{y})))$. Thus, each sinogram pixel is predicted once per forward pass, enabling more effective use of contextual information and improving training efficiency.

To compensate for information excluded by blind-spot masking, we couple the masked branch with the non-blind branch using the following objective~\cite{wang2022blind2unblind}:
\begin{equation}
\underbrace{\left\| \hat{y}_{\mathrm{b}} + \lambda_{\mathrm{vis}} \operatorname{sg}[{f}_\theta(\tilde{y})] - (\lambda_{\mathrm{vis}} + 1) \tilde{y} \right\|_2^2}_{\text{Re-visibility Loss}} 
+ \underbrace{\eta \cdot \left\| \hat{y}_{\mathrm{b}} - \tilde{y} \right\|_2^2}_{\text{Regularizer}}
\label{eq:visible_blind_spot_loss}
\end{equation}
Here, $\operatorname{sg}[\cdot]$ denotes the stop-gradient operation and is treated as a constant during optimization. The parameters $\lambda_{\mathrm{vis}}$ and $\eta$ control the contribution of the visible branch and the strength of the regularization, respectively. The loss indirectly optimizes the visible-branch output through the blind-spot prediction $\hat{y}_{\mathrm{b}}$. As $\lambda_{\mathrm{vis}}$ increases during training, the network gradually transitions from blind-spot constraints to non-blind inputs, making better use of visible information while preserving details. The regularizer constrains $\hat{y}_{\mathrm{b}}$ and limits blind-spot error propagation into the non-blind branch.

Regarding the masking ratio, a high ratio limits context and denoising capacity, whereas a low ratio weakens constraint and slows convergence. We therefore use high and low masking ratios $\rho_h$ and $\rho_l$. Inputs $\Omega_h(\tilde{y})$ and $\Omega_l(\tilde{y})$ produce predictions $\hat{y}_{\mathrm{b}}^{(h)}$ and $\hat{y}_{\mathrm{b}}^{(l)}$, with losses $\mathcal{L}_h$ and $\mathcal{L}_l$ from Eq.~\eqref{eq:visible_blind_spot_loss}. The two branches are combined with the following training objective:
\begin{equation}
\mathcal{L}_{\mathrm{prior}}
=
\left[1-\beta\right]\mathcal{L}_{h}
+
\beta\mathcal{L}_{l},
\label{eq:prior_loss}
\end{equation}
The schedule for $\beta$ shifts emphasis from the stronger blind-spot constraint of the high-ratio branch to the richer context of the low-ratio branch. Early optimization emphasizes robust masked prediction, while later optimization favors detail recovery. The sinogram prior is
\begin{equation}
\hat{y}=f_{\theta}(\tilde{y}).
\label{eq:prior_generate}
\end{equation}

\subsection{Posterior-Guided Pair Construction}
\label{sec:posterior_compensation}

Fig.~\ref{fig:framework}(b) illustrates the construction of two branches from one low-dose observation and the sinogram prior $\hat{y}$. The procedure includes photon-count posterior inference, prior correction, Poisson--Gaussian component decomposition, branch generation, residual scaling, and image-domain reconstruction.

According to the Beer--Lambert law, for an ideal line-integral sinogram $y$, the expected photon count is given by
\begin{equation}
\mu = I_0\exp(-y),
\label{eq:photon_count}
\end{equation}
where $I_0$ is the unattenuated incident photon count. Because $Y$ combines the latent photon count and Gaussian readout noise in Eq.~\eqref{eq:poisson_gaussian_observation}, neither component is observed directly. The resulting noisy sinogram is $\tilde{y} = -\log\left(\frac{Y}{I_0}\right)$.

To infer the latent photon count, we substitute $\hat{y}$ into Eq.~\eqref{eq:photon_count} to obtain the photon-count prior $\hat{\mu}$. Combining this prior with the noise variance $\sigma^2$, Bayes' rule gives the posterior distribution of $C$ as
\begin{equation}
p\left(c\mid Y,\hat{\mu},\sigma^{2}\right)
\propto
\frac{\hat{\mu}^{c}\exp(-\hat{\mu})}{c!}
\exp\left[
-\frac{(Y-c)^2}{2\sigma^{2}}
\right].
\end{equation}
Here, $c$ is a possible realization of $C$. The two factors are the Poisson prior and Gaussian readout noise likelihood.

In practice, the posterior distribution is evaluated over a finite integer support $\mathcal{C}$ in the log domain. The normalized posterior probability is written as
\begin{equation}
\begin{aligned}
\pi(c)
&=
\frac{\exp[\ell(c)]}
{\sum_{c'\in\mathcal{C}}\exp[\ell(c')]},\\
\ell(c)
&=
c\log\hat{\mu}-\hat{\mu}-\log\Gamma(c+1)
-\frac{(Y-c)^2}{2\sigma^2}.
\end{aligned}
\label{eq:log_unnormalized_posterior}
\end{equation}
Here, $\mathcal{C}$ is the finite candidate set, $c'$ is a summation index, and $\Gamma(\cdot)$ is the Gamma function. The posterior mean $\bar{c}$ and variance $v_c$ summarize the count estimate and uncertainty.

Residual noise or bias in $\hat{\mu}$ can distort branch statistics during posterior sampling. We therefore use posterior uncertainty to decide whether to retain the network prior or move it toward the posterior estimate. Their normalized discrepancy is
\begin{equation}
\chi=
\frac{
\operatorname{Avg}\left[(\bar{c}-\hat{\mu})^2\right]
}{
\operatorname{Avg}\left[(v_c+\bar{c}+\sigma^2)/N_{\mathrm{eff}}\right]
}.
\label{eq:posterior_consistency_ratio}
\end{equation}
Here, $\operatorname{Avg}(\cdot)$ is a local average and $N_{\mathrm{eff}}$ is the effective sample count. The adaptive weight and corrected prior are
\begin{equation}
W=\left(1-\chi^{-1}\right)_+,
\qquad
\mu_{\mathrm{pc}}=\hat{\mu}+W(\bar{c}-\hat{\mu}).
\label{eq:corrected_mu}
\end{equation}
Here, $(z)_+=\max(z,0)$. The weight approaches zero when the network prior is statistically consistent with the posterior. Conversely, a larger discrepancy increases $W$ and moves the photon-count prior toward $\bar{c}$. This correction suppresses unreliable network estimates while retaining those consistent with the posterior statistics. The corrected sinogram prior is
\begin{equation}
\hat{y}_{\mathrm{pc}}
=
-\log\left(\frac{\mu_{\mathrm{pc}}}{I_0}\right).
\label{eq:corrected_sinogram_prior}
\end{equation}

Posterior recomputation with $\mu_{\mathrm{pc}}$ gives
\begin{equation}
\begin{gathered}
C^*\sim p(c\mid Y,\mu_{\mathrm{pc}},\sigma^2),\quad
\varepsilon^*=Y-C^*.
\end{gathered}
\end{equation}
We then apply distribution-specific data thinning to the inferred
Poisson and Gaussian components~\cite{neufeld2024data}. For the Poisson component, binomial thinning produces two count branches:
\begin{equation}
C^{(1)}\mid C^*
\sim
\operatorname{Binomial}(C^*,\alpha),
\qquad
C^{(2)}=C^*-C^{(1)},
\end{equation}
where $\alpha\in(0,1)$ is the thinning ratio. Under the ideal Poisson model and exact posterior sampling, this procedure yields branches that are independent conditional on $y$, with effective dose fractions $\alpha$ and $1-\alpha$.

For the Gaussian component, we draw $Z\sim\mathcal{N}(0,\sigma^2)$ independently of $\varepsilon$. The two branches are
\begin{equation}
\begin{aligned}
\varepsilon^{(1)}
=
\alpha\varepsilon^*
+\tau Z,\quad
\varepsilon^{(2)}
=
(1-\alpha)\varepsilon^*
-\tau Z,
\end{aligned}
\end{equation}
where $\tau$ controls inter-branch covariance. We set $\tau=\sqrt{\alpha(1-\alpha)}$, making each Gaussian variance proportional to its effective dose fraction and giving zero cross-covariance under the ideal Gaussian model. Together with binomial thinning, it yields mixed-observation branches that are independent conditional on $y$ under the exact Poisson--Gaussian model and exact posterior sampling. 

The two components are subsequently recombined to form
\begin{equation}
\begin{aligned}
Y^{(k)}
=
C^{(k)}+\varepsilon^{(k)},\quad
\tilde{y}^{(k)}
=
-\log\left(
\frac{Y^{(k)}}{d_kI_0}
\right),
\end{aligned}
\label{eq:phy_sinogram_branch}
\end{equation}
where $k\in\{1,2\}$, and $d_1=\alpha$, $d_2=1-\alpha$ denote the effective dose fractions.

A first-order approximation of the logarithmic transform gives the conditional noise variance of branch $k$ as
\begin{equation}
\operatorname{Var}
\left(
\tilde{y}^{(k)}\mid y
\right)
\approx
\frac{1}{d_k}
\operatorname{Var}
\left(
\tilde{y}\mid y
\right).
\end{equation}

The logarithmic transform increases a branch's conditional noise variance by a factor of approximately $1/d_k$. Reconstruction would therefore train the image-domain network on noisier inputs than those encountered during inference. We avoid this distribution shift by scaling each branch residual around $\hat{y}_{\mathrm{pc}}$ to match the observation variance:
\begin{equation}
\tilde{y}_{\mathrm{vm}}^{(k)}
=
\hat{y}_{\mathrm{pc}}
+
\alpha_k^{\mathrm{var}}
\left(
\tilde{y}^{(k)}-\hat{y}_{\mathrm{pc}}
\right).
\label{eq:branch_sinogram_correction}
\end{equation}
Here, $\alpha_k^{\mathrm{var}}$ is the variance-matching coefficient and $\tilde{y}_{\mathrm{vm}}^{(k)}$ is the scaled sinogram. Matching the conditional branch variance to that of the observation gives $\alpha_k^{\mathrm{var}}=\sqrt{d_k}$. Only the residual amplitude changes, preserving stochastic branch generation. The scaled branches retain separate random components at the noise level used for inference. The Appendix derives the variance-matching coefficient, establishes the conditional independence of the split branches under the stated ideal assumptions, and further analyzes the effects of practical posterior and prior errors.

We reconstruct the scaled sinograms in the image domain:
\begin{equation}
x^{(k)}
=
A^{-1}
\left(
\tilde{y}_{\mathrm{vm}}^{(k)}
\right),
\label{eq:recon_branch}
\end{equation}
where $A^{-1}(\cdot)$ denotes the CT reconstruction operator and $x^{(1)}$ and $x^{(2)}$ are the reconstructed branch images. The resulting pair is used for image-domain self-supervised training.

\subsection{Cross-Domain Iterative Training}
\label{subsec:dual_domain}

Fig.~\ref{fig:framework}(c) illustrates the training process of cross-domain iterative refinement. The corrected prior from Sec.~\ref{sec:posterior_compensation} initializes the process: $\hat{y}_{0}=\hat{y}_{\mathrm{pc}}$.

The generator $\mathcal{G}$ maps $\hat{y}_t$ to two stochastic images:
\begin{equation}
    \left(x_{t}^{(1)},x_{t}^{(2)}\right)
    =
    \mathcal{G}\!\left(\hat{y}_{t},\xi_{t}\right),
    \label{eq:sample_generation}
\end{equation}
where $\xi_t$ represents randomness in posterior sampling and pair construction. The image-domain network $g_{\omega}$ is optimized by
\begin{equation}
\begin{gathered}
\mathcal{L}_{\mathrm{img}}
={}
\left\|
    g_{\omega}\!\left(x_{t}^{(1)}\right)
    -
    x_{t}^{(2)}
\right\|_{2}^{2}
+
\left\|
    g_{\omega}\!\left(x_{t}^{(2)}\right)
    -
    x_{t}^{(1)}
\right\|_{2}^{2},\\
\omega_{t}
\approx
\arg\min_{\omega}
\mathcal{L}_{\mathrm{img}}.
\end{gathered}
\label{eq:image_update}
\end{equation}

The restored estimates are averaged and forward-projected:
\begin{equation}
y_{t+1}^{\mathrm{img}} = A\left[\big(g_{\omega_{t}}\!\left(x_{t}^{(1)}\right)+g_{\omega_{t}}\!\left(x_{t}^{(2)}\right)\big)/2\right], \label{eq:image_reprojection}
\end{equation}
where $A(\cdot)$ is the forward projector. The reference $y_{t+1}^{\mathrm{img}}$ from the restored images enters the sinogram update through
\begin{equation}
\begin{gathered}
\mathcal{L}_{\mathrm{sino}}
=
\mathcal{L}_{\mathrm{prior}}
+
\lambda_{\mathrm{iter}}
\left\|
f_{\theta}(\tilde{y})
-
y_{t+1}^{\mathrm{img}}
\right\|_{2}^{2},
\\
\theta_{t+1}
\approx
\arg\min_{\theta}
\mathcal{L}_{\mathrm{sino}}.
\end{gathered}
\label{eq:sinogram_update}
\end{equation}
Here, $\lambda_{\mathrm{iter}}$ weights cross-domain consistency and $\theta_{t+1}$ denotes the updated parameters. The prior for the next outer iteration is defined as $\hat{y}_{t+1}=f_{\theta_{t+1}}(\tilde{y})$.

The prior conditions the image-domain pair distribution, and reprojected outputs provide a target for updating it. Each outer iteration transfers information from the projection domain to the image domain and back, reducing prior error and promoting cross-domain consistency.

\begin{figure}[t]
    \centering
    \includegraphics[width=\columnwidth]{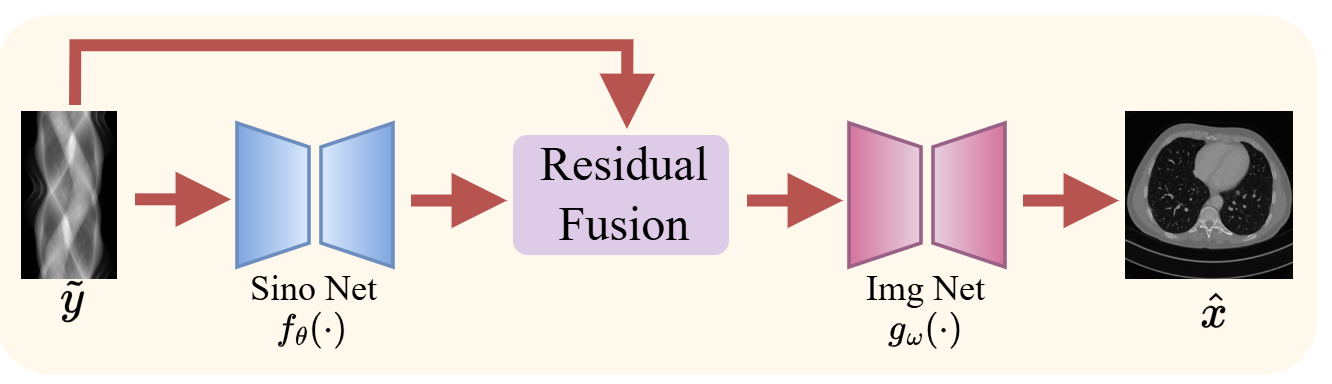}
    \caption{Inference pipeline with sinogram residual fusion and image refinement.}
    \label{fig:inference}
\end{figure}

\subsection{Inference and Reconstruction}
\label{subsec:inference}

At inference, the optimized sinogram network $f_{\theta^*}$ first estimates a prior from $\tilde{y}$, as shown in Fig.~\ref{fig:inference}:
\begin{equation}
\hat{y}=f_{\theta^*}(\tilde{y}).
\label{eq:inference_prior}
\end{equation}
The denoised prior may attenuate weak structures together with noise. We retain measurement information by fusing the prior with the observation residual:
\begin{equation}
\tilde{y}_{\mathrm{inf}}
=
\hat{y}
+
\alpha_{\mathrm{inf}}
\left(\tilde{y}-\hat{y}\right),
\label{eq:inference_fusion}
\end{equation}
where $\alpha_{\mathrm{inf}}$ balances noise suppression and detail preservation. The residual term reintroduces information attenuated by the sinogram network, while the prior provides the primary denoised estimate. The optimized image-domain network $g_{\omega^*}$ produces the final CT image $\hat{x}$ from the fused sinogram:
\begin{equation}
\hat{x}
=
g_{\omega^*}
\left(
A^{-1}\left(\tilde{y}_{\mathrm{inf}}\right)
\right).
\label{eq:inference_output}
\end{equation}

\begin{algorithm}[t]
\caption{Training and Inference Procedure}
\label{alg:framework}

\begin{algorithmic}[1]

\Statex \hspace{-\labelwidth}\hspace{-\labelsep}
\textbf{Input:}\ Noisy sinogram $\tilde{y}$, incident count $I_0$, readout variance

\Statex \hspace{-\labelwidth}\hspace{-\labelsep}
\phantom{\textbf{Input:}}\ $\sigma^2$, and outer iteration count $T$

\Statex \hspace{-\labelwidth}\hspace{-\labelsep}
\textbf{Output:}\ Denoised CT image $\hat{x}$

\Statex \hspace{-\labelwidth}\hspace{-\labelsep}
\textbf{Training Phase:}

\For{$t=1,\ldots,T$}

    \State Update $f_\theta$ using Eq.~\eqref{eq:prior_loss} if $t=1$,
    otherwise Eq.~\eqref{eq:sinogram_update}

    \State Correct the current sinogram prior using
    Eqs.~\eqref{eq:prior_generate} and
    \eqref{eq:log_unnormalized_posterior}--\eqref{eq:corrected_sinogram_prior}

    \For{each mini-batch}
        \State Resample $(x_t^{(1)},x_t^{(2)})$
        using Eq.~\eqref{eq:sample_generation}
        \State Update $g_\omega$ using Eq.~\eqref{eq:image_update}
    \EndFor

    \State Reproject the restored images using
    Eq.~\eqref{eq:image_reprojection}

\EndFor

\Statex \hspace{-\labelwidth}\hspace{-\labelsep}
\textbf{Inference Phase:}

\State Obtain $\tilde{y}_{\mathrm{inf}}$ using
Eqs.~\eqref{eq:inference_prior} and \eqref{eq:inference_fusion}

\State Obtain $\hat{x}$ using Eq.~\eqref{eq:inference_output}

\end{algorithmic}
\end{algorithm}

\section{Experiments}

The experimental evaluation consists of three parts: Experimental settings (Sec.~\ref{subsec:experimental_settings}), quantitative and qualitative comparisons on simulated and real low-dose CT data (Sec.~\ref{subsec:experimental_results}), and ablation analyses of the main components (Sec.~\ref{subsec:ablation_study}). 

\begin{table*}[htb]
    \centering
    \caption{Quantitative results on simulated AAPM, LIDC-IDRI, and LoDoPaB-CT data at dose levels of 0.5\% and 1.0\%.}
    \label{tab:sim_results}

    \setlength{\tabcolsep}{3pt}
    \renewcommand{\arraystretch}{1.05}

    \resizebox{\textwidth}{!}{
    \begin{tabular}{lcccccccccc}
        \toprule
        \multirow[c]{3}{*}[-1.6ex]{Type} &
        \multirow[c]{3}{*}[-1.6ex]{Method} &
        \multicolumn{9}{c}{$I_0=5{\times}10^{3}$, 0.5\% dose} \\
        \cmidrule(lr){3-11}
        & &
        \multicolumn{3}{c}{AAPM} &
        \multicolumn{3}{c}{LIDC-IDRI} &
        \multicolumn{3}{c}{LoDoPaB-CT} \\
        \cmidrule(lr){3-5} \cmidrule(lr){6-8} \cmidrule(lr){9-11}
        & &
        PSNR(dB) $\uparrow$ & SSIM(\%) $\uparrow$ & RMSE $\downarrow$ &
        PSNR(dB) $\uparrow$ & SSIM(\%) $\uparrow$ & RMSE $\downarrow$ &
        PSNR(dB) $\uparrow$ & SSIM(\%) $\uparrow$ & RMSE $\downarrow$ \\
        \midrule

        - & FBP
        & 24.95$\pm$1.09 & 49.77$\pm$2.91 & 57.73$\pm$7.02
        & 29.62$\pm$1.93 & 58.03$\pm$6.85 & 33.22$\pm$7.10
        & 28.15$\pm$2.33 & 60.85$\pm9.58$ & 42.23$\pm$11.24 \\

        Trad. & BM3D
        & 36.05$\pm$1.14 & 94.08$\pm$0.57 & 16.16$\pm$2.13
        & 37.68$\pm$1.93 & 94.70$\pm$1.30 & 14.15$\pm$2.94
        & 34.64$\pm$1.47 & 91.42$\pm$2.10 & 19.48$\pm$3.23 \\

        \multirow{6}{*}{Self.}
        & Noise2Void
        & 34.21$\pm$1.71 & 93.69$\pm$0.48 & 20.75$\pm$4.10
        & 35.59$\pm$1.97 & 93.86$\pm$1.41 & 18.10$\pm$3.77
        & 33.36$\pm$1.76 & 90.63$\pm$1.92 & 23.25$\pm$4.47 \\

        & Noise2Noise
        & 34.72$\pm$2.09 & 94.46$\pm$0.51 & 20.17$\pm$4.79
        & 36.56$\pm$2.90 & 95.57$\pm$1.08 & 17.50$\pm$5.69
        & 33.89$\pm$2.56 & 92.33$\pm$1.71 & 23.09$\pm$6.63 \\

        & Blind2Unblind
        & 36.36$\pm$1.12 & 94.13$\pm$0.44 & 15.64$\pm$2.01
        & 38.05$\pm$2.15 & 96.10$\pm$0.92 & 14.13$\pm$3.35
        & 35.21$\pm$2.04 & 92.77$\pm$1.74 & 19.37$\pm$4.39 \\

        & Noise2Sim
        & 37.82$\pm$1.22 & 95.40$\pm$0.42 & 13.23$\pm$1.86
        & 37.19$\pm$2.53 & 95.88$\pm$1.21 & 15.62$\pm$4.30
        & 34.10$\pm$2.34 & 92.55$\pm$1.92 & 22.12$\pm$5.79 \\

        & N2N-BS
        & 37.91$\pm$1.01 & 93.25$\pm$0.61 & 13.00$\pm$1.49
        & 40.20$\pm$1.87 & 95.01$\pm$1.31 & 10.41$\pm$2.12
        & 36.22$\pm$1.55 & 92.60$\pm$2.10 & 16.28$\pm$2.82 \\

        & Ours
        & \textbf{38.39$\pm$1.09} & \textbf{95.59$\pm$0.64} & \textbf{12.32$\pm$1.54}
        & \textbf{40.89$\pm$1.76} & \textbf{96.39$\pm$1.01} & \textbf{9.55$\pm$1.86}
        & \textbf{36.79$\pm$1.54} & \textbf{93.69$\pm$1.82} & \textbf{15.29$\pm$2.66} \\

        Sup. & RED-CNN
        & 37.77$\pm$1.22 & 95.26$\pm$0.45 & 13.27$\pm$1.87
        & 39.10$\pm$1.82 & 95.16$\pm$1.25 & 11.70$\pm$2.38
        & 35.98$\pm$1.57 & 92.21$\pm$1.99 & 16.73$\pm$2.98 \\

        \midrule

        \multirow[c]{3}{*}[-1.6ex]{Type} &
        \multirow[c]{3}{*}[-1.6ex]{Method} &
        \multicolumn{9}{c}{$I_0=1{\times}10^{4}$, 1.0\% dose} \\
        \cmidrule(lr){3-11}
        & &
        \multicolumn{3}{c}{AAPM} &
        \multicolumn{3}{c}{LIDC-IDRI} &
        \multicolumn{3}{c}{LoDoPaB-CT} \\
        \cmidrule(lr){3-5} \cmidrule(lr){6-8} \cmidrule(lr){9-11}
        & &
        PSNR(dB) $\uparrow$ & SSIM(\%) $\uparrow$ & RMSE $\downarrow$ &
        PSNR(dB) $\uparrow$ & SSIM(\%) $\uparrow$ & RMSE $\downarrow$ &
        PSNR(dB) $\uparrow$ & SSIM(\%) $\uparrow$ & RMSE $\downarrow$ \\
        \midrule

        - & FBP
        & 28.09$\pm$1.09 & 61.89$\pm$2.94 & 40.18$\pm$4.89
        & 30.67$\pm$1.77 & 62.59$\pm$5.20 & 30.20$\pm$5.82
        & 31.22$\pm$2.26 & 72.24$\pm$8.40 & 29.51$\pm$7.64 \\

        Trad. & BM3D
        & 37.16$\pm$1.40 & 94.70$\pm$0.73 & 14.29$\pm$2.31
        & 38.12$\pm$2.14 & 94.86$\pm$1.25 & 13.64$\pm$3.11
        & 36.05$\pm$1.50 & 92.12$\pm$2.10 & 16.55$\pm$2.76 \\

        \multirow{6}{*}{Self.}
        & Noise2Void
        & 35.76$\pm$1.41 & 93.58$\pm$0.54 & 17.04$\pm$2.69
        & 36.87$\pm$2.32 & 95.67$\pm$1.07 & 15.57$\pm$4.01
        & 34.60$\pm$1.87 & 92.85$\pm$1.68 & 20.10$\pm$4.15 \\

        & Noise2Noise
        & 36.72$\pm$1.66 & 95.31$\pm$0.40 & 15.76$\pm$3.00
        & 38.31$\pm$2.70 & 96.38$\pm$0.90 & 14.26$\pm$4.23
        & 35.43$\pm$2.60 & 93.59$\pm$1.51 & 19.55$\pm$5.70 \\
        
        & Blind2Unblind
        & 37.66$\pm$1.38 & 95.66$\pm$0.38 & 13.96$\pm$2.19
        & 39.07$\pm$1.95 & 96.72$\pm$0.84 & 12.47$\pm$2.63
        & 36.30$\pm$2.09 & 93.87$\pm$1.63 & 17.14$\pm$3.99 \\

        & Noise2Sim
        & 38.51$\pm$1.32 & 95.64$\pm$0.45 & 12.29$\pm$1.88
        & 38.75$\pm$2.38 & 96.56$\pm$1.07 & 12.97$\pm$3.31
        & 35.22$\pm$2.41 & 93.54$\pm$1.79 & 19.48$\pm$5.20 \\
        
        & N2N-BS
        & 39.17$\pm$0.93 & 94.20$\pm$0.70 & 11.18$\pm$1.19
        & 42.21$\pm$1.86 & 96.67$\pm$1.01 & 8.22$\pm$1.67
        & 37.73$\pm$1.61 & 94.12$\pm$1.88 & 13.67$\pm$2.44 \\

        & Ours
        & \textbf{39.79$\pm$1.01} & \textbf{95.72$\pm$0.55} & \textbf{10.44$\pm$1.21}
        & \textbf{42.62$\pm$1.81} & \textbf{97.32$\pm$0.81} & \textbf{7.84$\pm$1.54}
        & \textbf{38.30$\pm$1.63} & \textbf{94.77$\pm$1.65} & \textbf{12.77$\pm$2.33} \\

        Sup. & RED-CNN
        & 39.24$\pm$1.14 & 95.91$\pm$0.38 & 11.19$\pm$1.47
        & 40.65$\pm$1.87 & 96.06$\pm$1.11 & 9.80$\pm$2.02
        & 37.47$\pm$1.61 & 93.61$\pm$1.78 & 14.09$\pm$2.63 \\

        \bottomrule
    \end{tabular}
    }

\end{table*}

\begin{figure*}[!t]
    \centering
    \includegraphics[width=\textwidth]{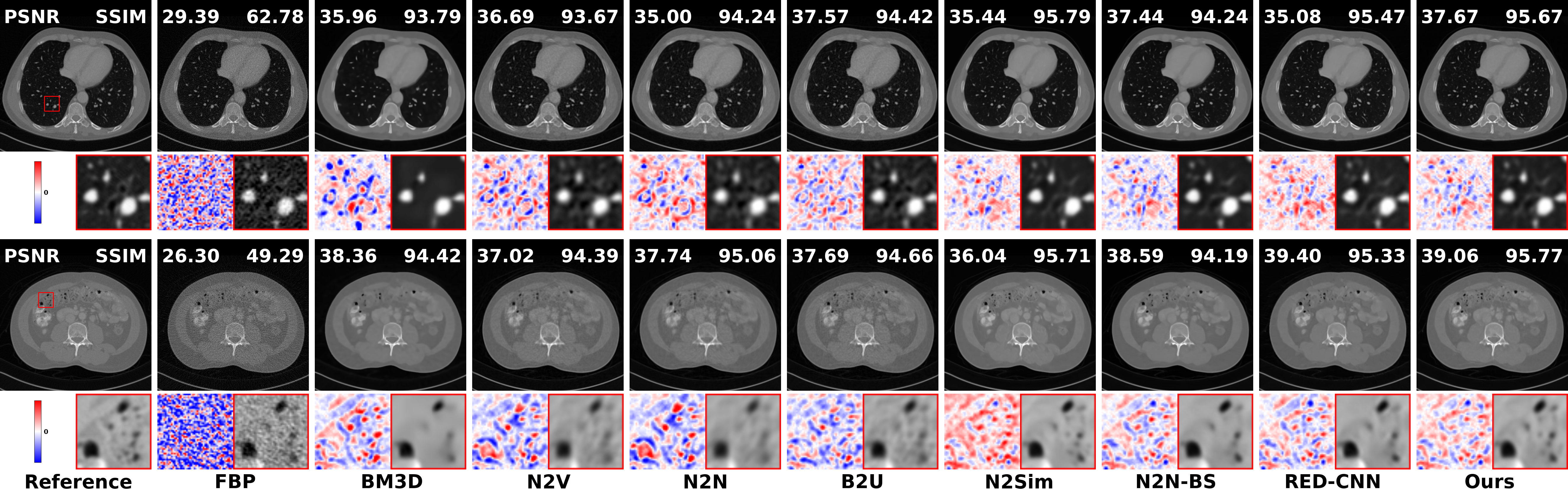}
    \caption{Representative denoising results on AAPM at $I_0=5\times10^{3}$ (0.5\% dose), with CT images, magnified ROIs, and error maps. ROIs are displayed using an HU window of $[-1000,1000]$ HU.}
    \label{fig:result_5000}
\end{figure*}

\begin{figure*}[htb]
    \centering
    \includegraphics[width=\textwidth]{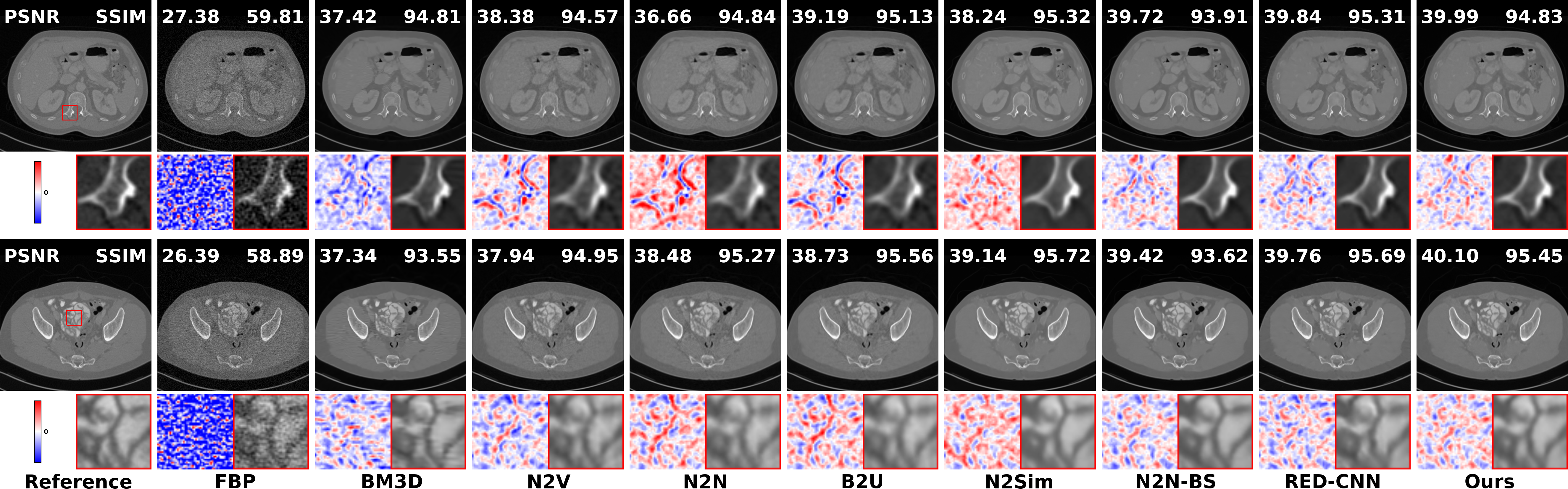}
    \caption{Representative denoising results on AAPM at $I_0=1\times10^{4}$ (1.0\% dose), with CT images, magnified ROIs, and error maps. ROIs are displayed using an HU window of $[-400,1600]$ HU.}
    \label{fig:result_10000}
\end{figure*}

\begin{figure*}[htb]
    \centering
    \includegraphics[width=\textwidth]{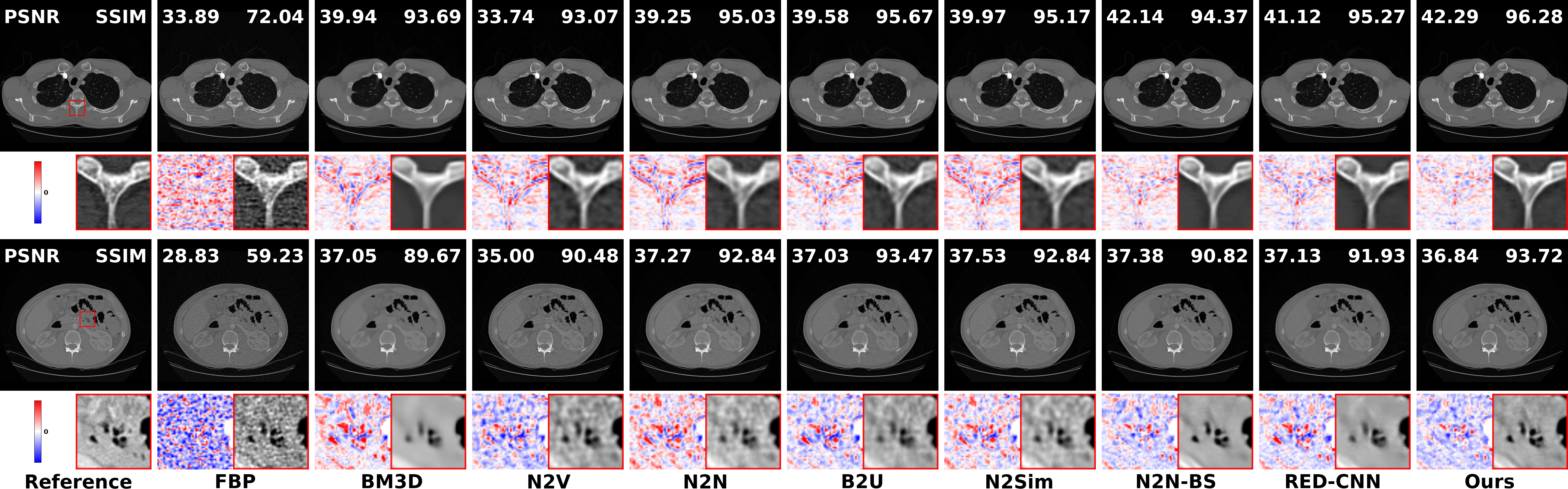}
    \caption{Representative denoising results on LIDC-IDRI at $I_0=5\times10^{3}$ (0.5\% dose), with CT images, magnified ROIs, and error maps. ROIs are displayed using an HU window of $[-400,1600]$ HU.}
    \label{fig:Li_result_5000}
\end{figure*}

\begin{figure*}[!t]
    \centering
    \includegraphics[width=\textwidth]{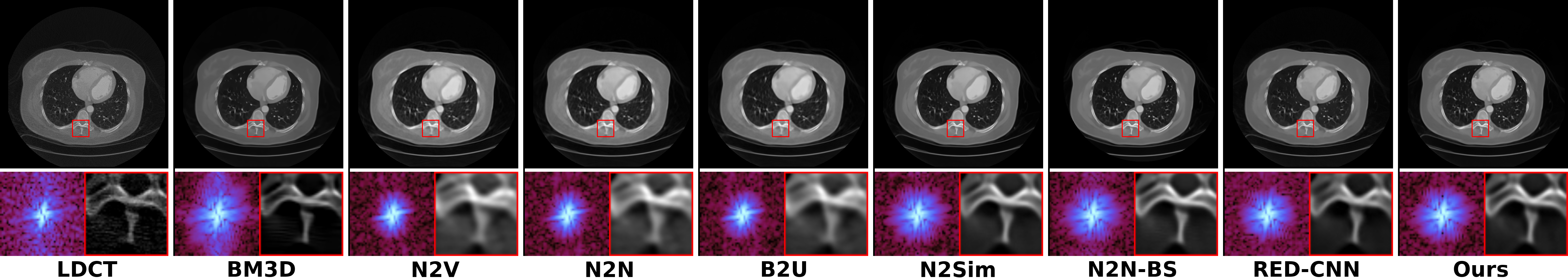}
    \caption{Representative denoising results on the GE Clinical Cardiac Dataset, with reconstructed images, local frequency spectra, and magnified ROIs. ROIs are displayed using an HU window of $[0,1600]$ HU.}
    \label{fig:result_GE}
\end{figure*}

\begin{figure*}[!t]
    \centering
    \includegraphics[width=\textwidth]{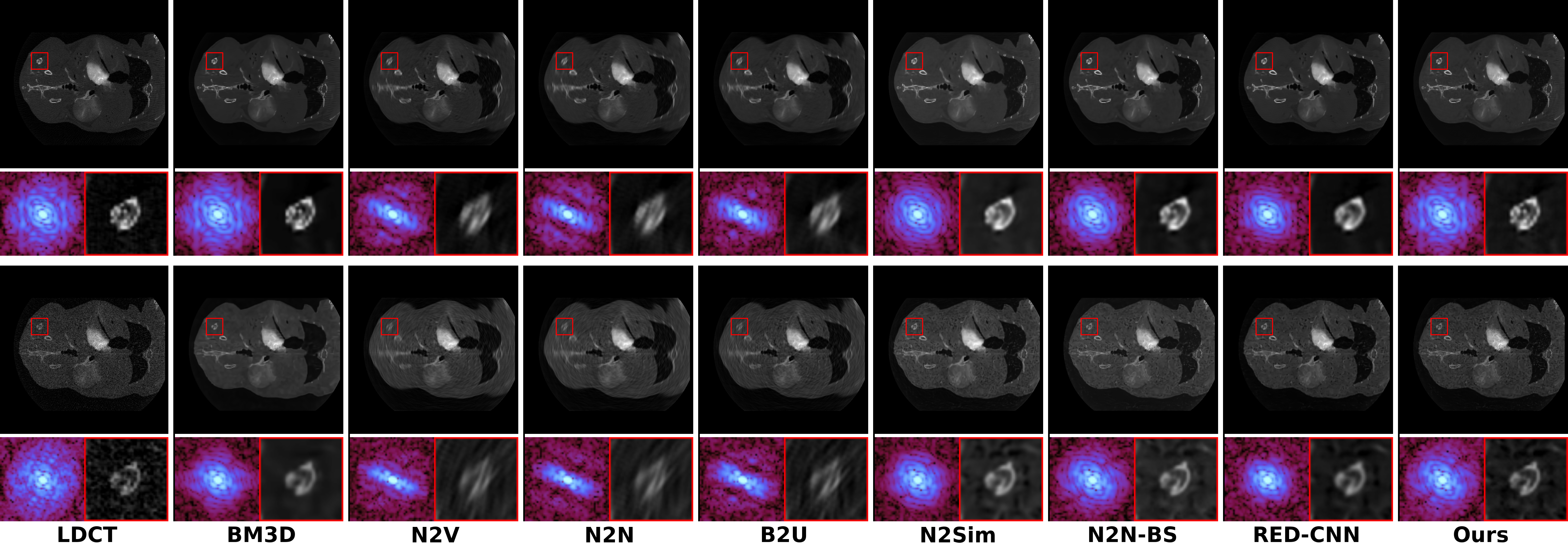}
    \caption{Representative denoising results on real low-dose mouse CT data, with reconstructed images, local frequency spectra, and magnified ROIs. ROIs are displayed using an HU window of $[0,1600]$ HU.}
    \label{fig:result_mouse}
\end{figure*}

\subsection{Experimental Settings}
\label{subsec:experimental_settings}

\subsubsection{Datasets}

Experiments were conducted on AAPM Mayo~\cite{mccollough2017lowdose}, LIDC-IDRI~\cite{armato2011lidc}, and LoDoPaB-CT~\cite{leuschner2021lodopab}. Patient-level splits contained 5,376 training and 560 test slices for AAPM and 5,733 training and 644 test slices for LIDC-IDRI. For cross-dataset evaluation, we applied the model trained on LIDC-IDRI to 700 LoDoPaB-CT images. Because raw projections were unavailable, we used ODL~\cite{adler2017odl} to forward project reference images in fan-beam geometry with 720 views and 512 detector channels. Measurements were simulated under the Poisson--Gaussian model at $I_0=5\times10^3$ (0.5\% dose) and $I_0=1\times10^4$ (1.0\% dose). 

Real-data experiments were conducted on the GE Clinical Cardiac Dataset and a mouse chest dataset. The GE dataset comprises high-resolution cardiac CT scans acquired with a 256-slice GE scanner under IRB approval at Vanderbilt University Medical Center and the University of Massachusetts Lowell. One clinical case with 984 projections over $360^\circ$ was evaluated. The detector contained 828 cells of $1.09\times1.09$~mm$^2$, with source-to-isocenter and source-to-detector distances of 625.61 and 1,097.6~mm, respectively. The mouse data were acquired at 90~kVp and 0.2~mA with 2,000 projections over $360^\circ$ using a photon-counting detector. The 20--90 and 60--90~keV energy windows were evaluated. The detector pixel size was $0.1\times0.1$~mm$^2$, with source-to-isocenter and source-to-detector distances of 76.28 and 361.1~mm, respectively. Each projection contained $2062\times252$ detector samples. The mouse data acquisition was approved by the Ethics Committee of the Li Ka Shing Faculty of Medicine, The University of Hong Kong (February 20, 2023).

\subsubsection{Competing Methods}

We compared FBP~\cite{kak2001principles}, BM3D~\cite{dabov2007bm3d}, supervised RED-CNN~\cite{chen2017redcnn}, and five self-supervised methods: Noise2Void (N2V)~\cite{krull2019noise2void}, Noise2Noise (N2N)~\cite{lehtinen2018noise2noise}, Blind2Unblind (B2U)~\cite{wang2022blind2unblind}, Noise2Sim (N2Sim)~\cite{niu2023noise2sim}, and photon-count-splitting N2N-BS~\cite{wang2025emulating}.

\subsubsection{Implementation Details}

All experiments were conducted on an NVIDIA RTX 5070 Ti GPU. Both networks used U-Net~\cite{ronneberger2015unet} and AdamW~\cite{loshchilov2019adamw} with a learning rate of $1.8\times10^{-4}$ and batch size 4. We set $\alpha=0.5$, $\sigma^2=5$, $\eta=1$, $\lambda_{\mathrm{iter}}=0.1$, $\alpha_{\mathrm{inf}}=1$, $\rho_h=1/4$, and $\rho_l=1/13$. During training, $\lambda_{\mathrm{vis}}$ increased from 2 to 13 and $\beta$ from 0.25 to 0.9. Each of five cross-domain iterations comprised 10 sinogram- and 20 image-domain epochs. The source code is available at \url{https://github.com/yqx7150/PIPG-LDCT}.

\subsubsection{Evaluation Metrics}

PSNR, SSIM~\cite{wang2004ssim}, and RMSE are reported as the mean $\pm$ standard deviation over the test set. Qualitative comparisons use identical CT windows, full slices, enlarged regions of interest (ROIs), and error maps to assess
artifact suppression and structural preservation.

\subsection{Experimental Results}
\label{subsec:experimental_results}

\subsubsection{Results on Simulated Data}

Table~\ref{tab:sim_results} indicates that the proposed method achieves the best PSNR, SSIM, and RMSE among the self-supervised methods on all datasets at both dose levels. Among these methods, N2Sim is sensitive to slice quality, whereas N2N-BS neglects the Gaussian noise component, resulting in degraded performance. The proposed method also outperforms supervised RED-CNN in 17 of 18 entries. Only AAPM SSIM at 1.0\% dose is slightly lower, at 95.72\% versus 95.91\%. 

Figs.~\ref{fig:result_5000}--\ref{fig:Li_result_5000} show that BM3D still leaves noticeable residual noise. N2V, N2N, and B2U further reduce the noise level. N2Sim, N2N-BS, and RED-CNN produce cleaner reconstructions overall, but exhibit oversmoothing in fine-detail recovery. In contrast, the proposed method suppresses noise and streak artifacts while better preserving low-contrast boundaries, high-frequency structures, and intensity transitions between tissues. The proposed method also produces more uniform error maps with fewer strong residual regions, particularly around high-contrast edges and low-contrast details.

\subsubsection{Results on Real Data}

Fig.~\ref{fig:result_GE} shows that BM3D still exhibits noticeable residual noise. N2V, N2N, and B2U introduce varying degrees of blurring around local edges and fine structures. N2Sim, N2N-BS, and RED-CNN improve visual quality, but some fine boundaries are still weakened. In contrast, the proposed method produces a cleaner and more uniform background while preserving structural edges and local intensity variations within the ROI. In the local frequency-domain amplitude spectrum, the proposed method suppresses disordered frequency components associated with random noise while retaining clearer high-frequency responses. Compared with the oversmoothed results, the high-frequency information is less attenuated. The consistent observations in the spatial and frequency domains indicate that the proposed method achieves a better balance between noise suppression and structural preservation under real acquisition noise. Fig.~\ref{fig:result_mouse} further demonstrates the stability and effectiveness of the proposed method on real low-dose mouse CT data. The consistent performance on cardiac and mouse CT data indicates its applicability to real low-dose CT data.

\begin{table}[htb]
\centering
\caption{Component ablation on AAPM at $I_0=1\times10^4$.}
\label{tab:ablation}
\renewcommand{\arraystretch}{1.12}
\setlength{\tabcolsep}{0pt}
\begin{tabular*}{\columnwidth}{@{\extracolsep{\fill}}lccc@{}}
\toprule
Variant &
PSNR (dB) $\uparrow$ &
SSIM (\%) $\uparrow$ &
RMSE $\downarrow$ \\
\midrule
w/o Img-domain    & 37.83$\pm$1.24 & 95.70$\pm$0.37 & 13.60$\pm$1.89 \\
w/o Posterior     & 39.17$\pm$0.93 & 94.20$\pm$0.70 & 11.18$\pm$1.19 \\
w/o Variance Matching & 39.39$\pm$1.09 & 95.11$\pm$0.80 & 10.94$\pm$1.37 \\
w/o Iteration     & 39.66$\pm$1.02 & 95.49$\pm$0.62 & 10.59$\pm$1.24 \\
\midrule
Full model         & \textbf{39.79$\pm$1.01} 
                   & \textbf{95.72$\pm$0.55} 
                   & \textbf{10.44$\pm$1.21} \\
\bottomrule
\end{tabular*}
\end{table}

\subsection{Ablation Study}
\label{subsec:ablation_study}

\begin{figure}[htb]
    \centering
    \includegraphics[width=\columnwidth]
    {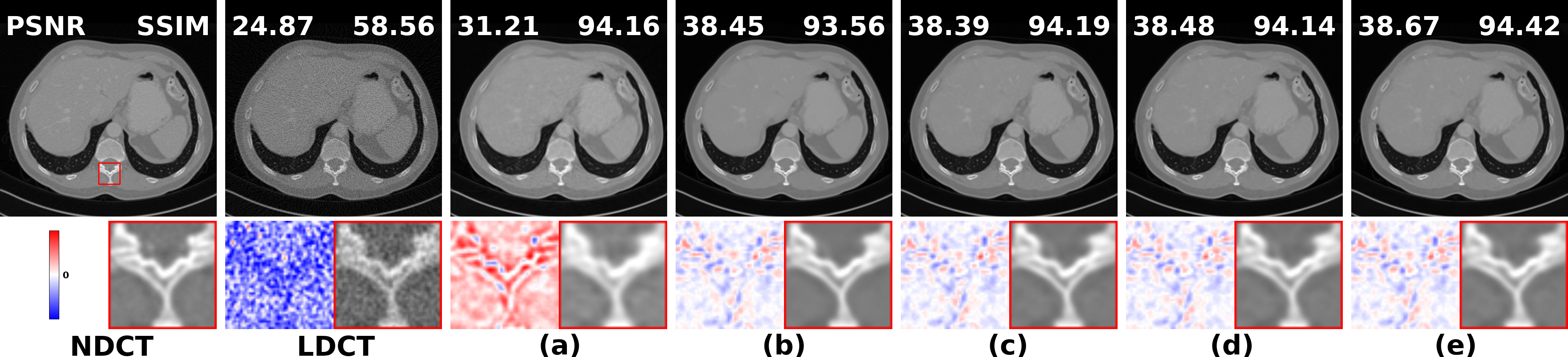}
    \caption{Qualitative component ablation on AAPM at $I_0=1\times10^4$. NDCT and LDCT denote the normal-dose reference and low-dose input. (a) w/o Img-domain. (b) w/o Posterior. (c) w/o Variance Matching. (d) w/o Iteration. (e) Full model. Each variant shows its reconstruction, magnified ROI, and error map.}
    \label{fig:altion_result_10000}
\end{figure}

\subsubsection{Component Ablation}

Table~\ref{tab:ablation} and Fig.~\ref{fig:altion_result_10000} assess the main components on AAPM at $I_0=1\times10^4$. Removing the image-domain network causes the largest decrease in PSNR and the largest increase in RMSE, indicating the importance of image-domain refinement. The \textit{w/o Posterior} variant omits latent count inference and explicit Gaussian noise separation, reducing it to the N2N-BS baseline. Results indicate direct Poisson splitting is unsuitable for mixed LDCT noise in this setting. Removing variance matching degrades all metrics, supporting the need to correct the noise mismatch induced by splitting. The \textit{w/o Iteration} variant retains sinogram-to-image processing. It removes image-to-sinogram feedback and prior updates. Its degradation supports cross-domain iteration.

\begin{figure}[!h]
    \centering
    \includegraphics[width=\columnwidth]{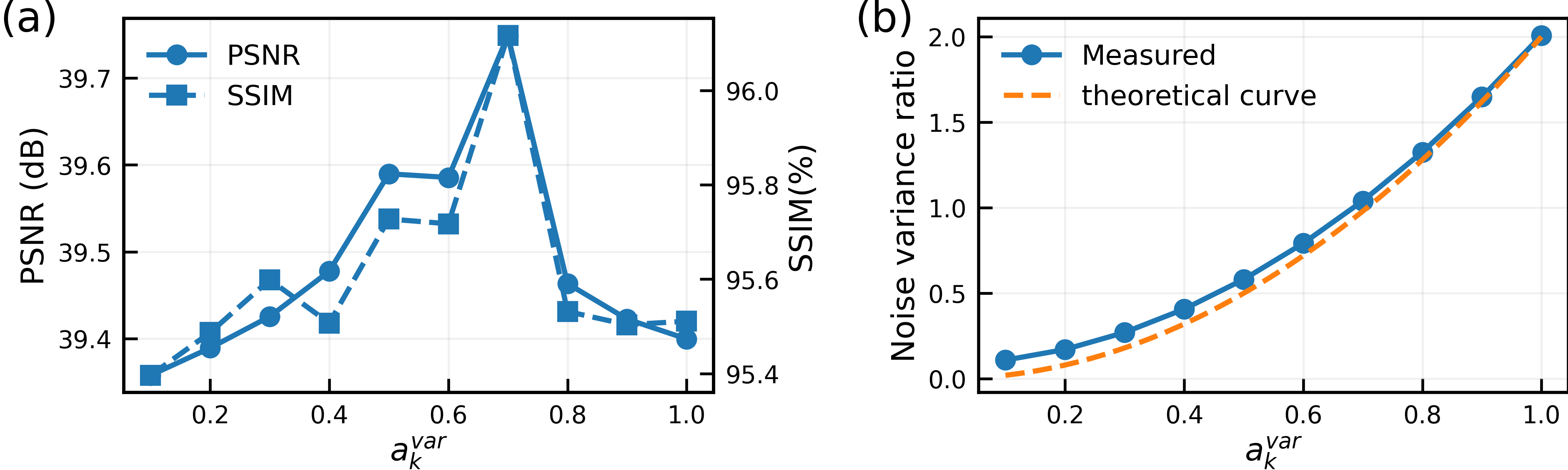}
    \caption{Variance-matching coefficient $\alpha_k^{\mathrm{var}}$ on AAPM at $I_0 = 1\times10^4$. (a) Reconstruction performance. (b) Measured and predicted ratios of branch noise variance to observation noise variance.}
    \label{fig:alpha_analysis}
\end{figure}

\subsubsection{Variance-Matching Analysis}

Fig.~\ref{fig:alpha_analysis}(a) shows that reconstruction performance peaks near $\alpha_k^{\mathrm{var}}=0.7$. Sec.~\ref{sec:posterior_compensation} also approximates the ratio of branch noise variance to observation noise variance as $(\alpha_k^{\mathrm{var}})^2/d_k$. Balanced splitting with $d_k=0.5$ therefore requires $\alpha_k^{\mathrm{var}}=\sqrt{0.5}\approx0.707$. In Fig.~\ref{fig:alpha_analysis}(b), the measured ratio closely follows the theoretical curve. The predicted matching point also lies in the region of optimal reconstruction. Agreement between the empirical optimum and analytical value supports the variance-matching mechanism.

\begin{table}[!h]
\centering
\caption{Effect of the thinning ratio $\alpha$ on AAPM at $I_0=1\times10^4$.}
\label{tab:splitting_ratio}
\renewcommand{\arraystretch}{1.12}
\setlength{\tabcolsep}{0pt}
\begin{tabular*}{\columnwidth}{@{\extracolsep{\fill}}cccc@{}}
\toprule
Branch fractions $(\alpha,1-\alpha)$  &
PSNR (dB) $\uparrow$ &
SSIM (\%) $\uparrow$ &
RMSE $\downarrow$ \\
\midrule
$(0.1,\,0.9)$ & 39.48$\pm$1.05 & 95.55$\pm$0.71 & 10.91$\pm$1.28 \\
$(0.3,\,0.7)$ & 39.67$\pm$0.99 & 95.68$\pm$0.58 & 10.60$\pm$1.21 \\
$(0.5,\,0.5)$ & \textbf{39.79$\pm$1.01} & \textbf{95.72$\pm$0.55} & \textbf{10.44$\pm$1.21} \\
\bottomrule
\end{tabular*}
\end{table}

\subsubsection{Effect of the Thinning Ratio}

Table~\ref{tab:splitting_ratio} shows that reconstruction performance is highest for the balanced split $(0.5,0.5)$ with equal branch fractions. Unequal splitting yields mismatched photon-count statistics and asymmetric branch predictions. Together with Fig.~\ref{fig:alpha_analysis}, this supports variance matching and balanced branches.

\begin{figure}[htb]
    \centering
    \includegraphics[width=\columnwidth]{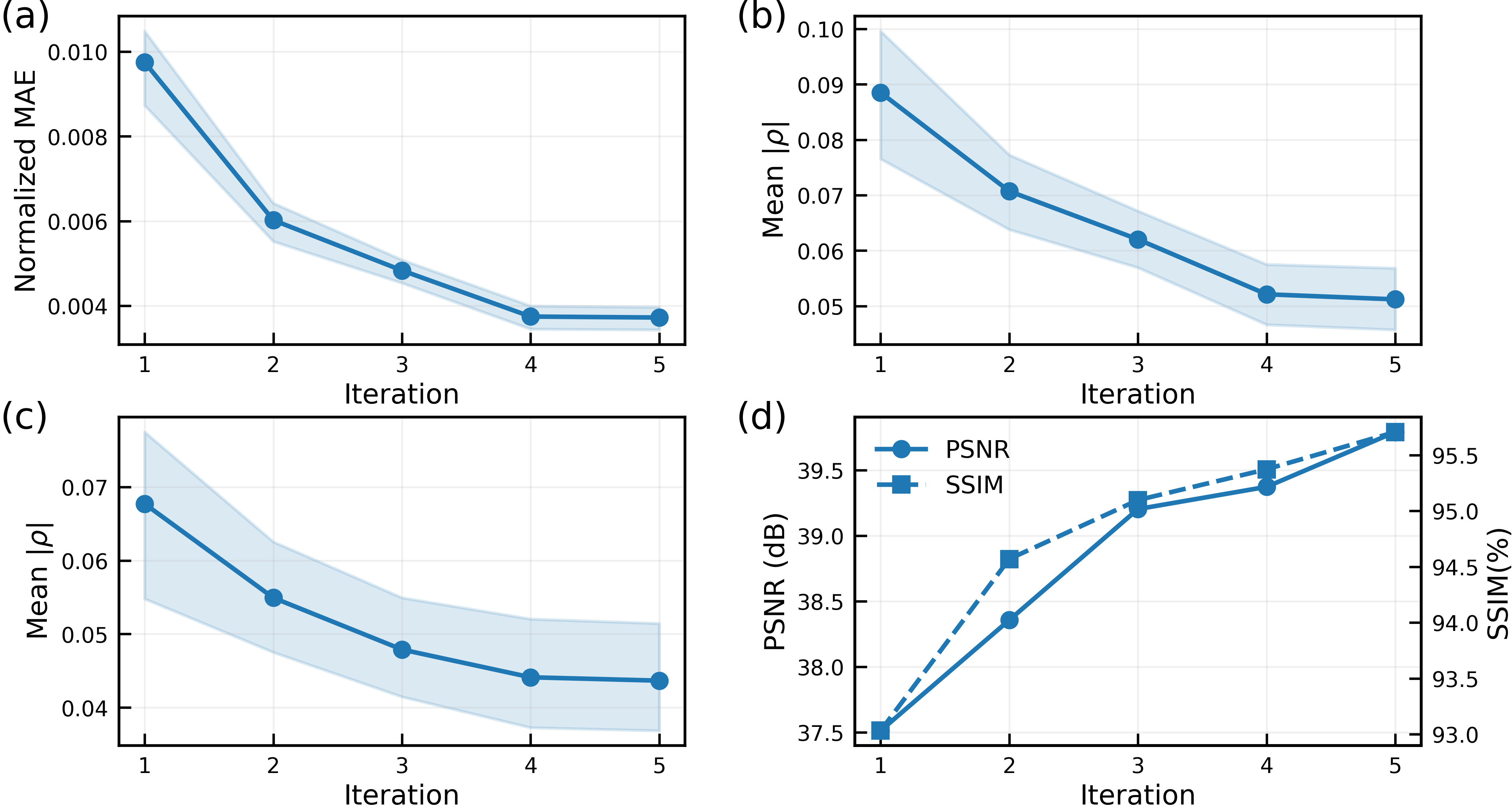}
    \caption{Effect of cross-domain iterative refinement on AAPM at $I_0=1\times10^4$. (a) Normalized sinogram MAE. (b)-(c) Mean absolute inter-branch residual correlation in the sinogram and image domains, respectively. (d) Reconstruction PSNR and SSIM across iterations.}
    \label{fig:iteration_analysis}
\end{figure}

\subsubsection{Cross-Domain Iterative Refinement}

Fig.~\ref{fig:iteration_analysis}(a) shows that normalized sinogram MAE decreases during iteration as image-domain feedback refines the prior. Mean absolute inter-branch residual correlation, computed relative to the clean references, also decreases in the sinogram and image domains in Figs.~\ref{fig:iteration_analysis}(b) and (c). This reduction indicates weaker empirical dependence between regenerated branches. PSNR and SSIM also increase in Fig.~\ref{fig:iteration_analysis}(d). Together, these trends link prior refinement and pair quality to improved reconstruction performance.

\section{Discussion}

The core of this work is to unify physics-driven training-pair construction with cross-domain iteration. Unlike methods that construct training pairs solely from the observed data, our method exploits the noise formation process in the projection domain to generate physically grounded training pairs, while continuously refining the sinogram prior using the restored image-domain results. By integrating these two components, training-pair construction is no longer limited to a static procedure based on fixed observations, but becomes a dynamic modeling process guided by imaging physics and progressively updated through cross-domain restoration, thereby extending the way training pairs can be constructed for self-supervised low-dose CT denoising.

Theoretical analysis further characterizes the conditions under which the proposed training-pair construction is valid, as well as its practical limitations. Under an ideal Poisson--Gaussian noise model and exact posterior sampling, the two generated branches can satisfy strict conditional independence. In practice, this ideal conditional independence may be compromised by posterior approximation and prior-estimation errors. When the actual CT acquisition process deviates substantially from the assumed Poisson--Gaussian model, for example due to scattering and beam-hardening effects, non-negligible residual correlations may remain between the two branches, which constitutes the main theoretical limitation of the current method. In addition, the present implementation is restricted to two-dimensional processing. 

Future work will investigate more comprehensive noise and acquisition models to better accommodate practical system deviations and explicitly account for model uncertainty in training-pair construction. Another promising direction is to integrate reconstruction and cross-domain refinement into a unified end-to-end or unrolled framework, allowing projection- and image-domain information to be jointly optimized while reducing the computational overhead of the current iterative procedure. We will also extend the present two-dimensional implementation to three-dimensional or multi-slice processing. Beyond LDCT, the proposed physics-driven pair-generation principle may be further generalized to other denoising or inverse problems governed by mixed Poisson--Gaussian noise.

\section{Conclusion}

This work proposed a self-supervised denoising method for low-dose CT that integrated physics-driven training-pair construction with cross-domain iteration to construct approximately independent training pairs from a single low-dose observation. Theoretical analysis established the ideal conditions under which the generated training pairs were conditionally independent and examined the effects of practical errors on their statistical properties. Experiments demonstrated the effectiveness of the proposed method and its potential for self-supervised low-dose CT denoising.

\appendix

To further characterize the ideal properties of the proposed pair construction and its deviations in practice, we present the following propositions and their proofs.

\begin{appproposition}
\label{prop:app_ideal}
Under the Poisson--Gaussian observation model, if
$C^* \sim p(C \mid Y,y)$, the two constructed raw branches satisfy $Y^{(1)} \perp Y^{(2)} \mid y$, and $\mathbb{E}\!\left[Y^{(k)} \mid y\right] = d_k \mu$, $\operatorname{Var}\!\left(Y^{(k)} \mid y\right)
= d_k(\mu+\sigma^2)$, where $d_1=\alpha$ and $d_2=1-\alpha$.
Furthermore, under the first-order logarithmic approximation,
$\operatorname{Var}\!\left(\tilde y^{(k)} \mid y\right)
\approx
\frac{1}{d_k}
\operatorname{Var}\!\left(\tilde y \mid y\right)$, and hence the variance-matching coefficient is $\alpha_k^{\mathrm{var}}=\sqrt{d_k}$.
\end{appproposition}

\begin{IEEEproof}
Conditional on $y$, $C\sim\operatorname{Poisson}(\mu)$ and
$\varepsilon\sim\mathcal{N}(0,\sigma^2)$ are independent. Let
$\varepsilon^*=Y-C^*$. Since $C^*$ is drawn from $p(C\mid Y,y)$, Bayes' rule gives
\begin{equation}
f_{Y\mid y}(c+e)p(C=c\mid Y=c+e,y)
=
p(C=c\mid y)f_{\varepsilon}(e).
\end{equation}
Thus, $(C^*,\varepsilon^*)\mid y\overset{d}{=}(C,\varepsilon)\mid y$, so
$C^*\sim\operatorname{Poisson}(\mu)$,
$\varepsilon^*\sim\mathcal{N}(0,\sigma^2)$, and
$C^*\perp\varepsilon^*\mid y$.

For the Poisson component, binomial thinning yields
$C^{(k)}\sim\operatorname{Poisson}(d_k\mu)$ and
$C^{(1)}\perp C^{(2)}\mid y$.

For the Gaussian component, with
$\tau=\sqrt{\alpha(1-\alpha)}$,
\begin{equation}
\operatorname{Cov}(\varepsilon^{(1)},\varepsilon^{(2)}\mid y)
=
[\alpha(1-\alpha)-\tau^2]\sigma^2
=
0.
\end{equation}
Since the pair is jointly Gaussian, this implies
$\varepsilon^{(1)}\perp\varepsilon^{(2)}\mid y$.

Since the Poisson and Gaussian splitting procedures use independent
auxiliary randomness and $C^*\perp\varepsilon^*\mid y$,
$(C^{(1)},C^{(2)})\perp(\varepsilon^{(1)},\varepsilon^{(2)})\mid y$.
Hence, for $Y^{(k)}=C^{(k)}+\varepsilon^{(k)}$,
$Y^{(1)}\perp Y^{(2)}\mid y$. Moreover,
\begin{equation}
\mathbb{E}[Y^{(k)}\mid y]=d_k\mu,
\qquad
\operatorname{Var}(Y^{(k)}\mid y)=d_k(\mu+\sigma^2).
\end{equation}

For the logarithmic transform, let
$g_k(t)=-\log[t/(d_kI_0)]$. Since $\mu=I_0\exp(-y)$,
$g_k(d_k\mu)=y$ and $g_k'(d_k\mu)=-1/(d_k\mu)$.
Applying the first-order delta method yields
\begin{equation}
\begin{aligned}
\mathbb{E}[\tilde{y}^{(k)}\mid y]
&\approx y,\\
\operatorname{Var}(\tilde{y}^{(k)}\mid y)
&\approx
[g_k'(d_k\mu)]^2\operatorname{Var}(Y^{(k)}\mid y)
=
\frac{\mu+\sigma^2}{d_k\mu^2}.
\end{aligned}
\end{equation}

Similarly, for $\tilde{y}=-\log(Y/I_0)$,
$\operatorname{Var}(\tilde{y}\mid y)\approx(\mu+\sigma^2)/\mu^2$.
Hence,
\begin{equation}
\operatorname{Var}(\tilde{y}^{(k)}\mid y)
\approx
\frac{1}{d_k}\operatorname{Var}(\tilde{y}\mid y).
\end{equation}

Thus, variance matching requires
$(\alpha_k^{\mathrm{var}})^2/d_k=1$, yielding
$\alpha_k^{\mathrm{var}}=\sqrt{d_k}$.
\IEEEQEDoff
\end{IEEEproof}

\begin{appproposition}
\label{prop:app_practical}
Let $\Pi_Y$ and $Q_Y$ be the exact and practical posteriors, and $d_{\mathrm{TV}}$ the total variation distance. Define
\begin{equation}
\varepsilon_{\mathrm{post}}(y)
=
\mathbb{E}_{Y\mid y}
\left[
d_{\mathrm{TV}}(Q_Y,\Pi_Y)
\right].
\label{eq:app_posterior_error}
\end{equation}
Let $P_Q^{12}$ and $P_\Pi^{12}$ be the joint raw-branch distributions induced by $Q_Y$ and $\Pi_Y$, with marginals $P_Q^k$ and $P_\Pi^k$, $k=1,2$. Then
\begin{equation}
d_{\mathrm{TV}}
\left(
P_Q^{12},P_Q^1\otimes P_Q^2
\right)
\le
3\varepsilon_{\mathrm{post}}(y).
\label{eq:app_tv_bound}
\end{equation}
Moreover, with $m_Q(y)=\mathbb{E}[C^*\mid y]$,
\begin{equation}
\operatorname{Cov}(Y^{(1)},Y^{(2)}\mid y)
=
\alpha(1-\alpha)[\mu-m_Q(y)].
\label{eq:app_raw_cov}
\end{equation}

After variance matching, let
$r_k=\tilde y^{(k)}-y$,
$e_{\mathrm{pc}}=\hat y_{\mathrm{pc}}-y$,
$a_k=\alpha_k^{\mathrm{var}}$, and denote their conditional standard
deviations by $\sigma_k$ and $\sigma_{\mathrm{pc}}$, respectively. Then
\begin{equation}
\begin{aligned}
\left|
\operatorname{Cov}
(\tilde y_{\mathrm{vm}}^{(1)},
 \tilde y_{\mathrm{vm}}^{(2)}
 \mid y)
\right|
&\le
a_1a_2
|\operatorname{Cov}(r_1,r_2\mid y)| \\
&\quad+
\sigma_{\mathrm{pc}}
[a_1(1-a_2)\sigma_1
+(1-a_1)a_2\sigma_2] \\
&\quad+
(1-a_1)(1-a_2)\sigma_{\mathrm{pc}}^2.
\end{aligned}
\label{eq:app_vm_bound}
\end{equation}
\end{appproposition}

\begin{IEEEproof}
Posterior sampling followed by stochastic splitting defines a Markov
kernel. By contraction of total variation,
$d_{\mathrm{TV}}(P_Q^{12},P_\Pi^{12})
\le\varepsilon_{\mathrm{post}}(y)$ and
$d_{\mathrm{TV}}(P_Q^k,P_\Pi^k)
\le\varepsilon_{\mathrm{post}}(y)$.
Proposition~\ref{prop:app_ideal} gives
$P_\Pi^{12}=P_\Pi^1\otimes P_\Pi^2$; hence the triangle inequality and
the product-measure inequality yield
Eq.~\eqref{eq:app_tv_bound}.

Conditioned on $(C^*,\varepsilon^*)$,
\begin{equation}
\operatorname{Cov}(Y^{(1)},Y^{(2)}\mid C^*,\varepsilon^*)
=
-\alpha(1-\alpha)(C^*+\sigma^2),
\end{equation}
while the conditional means are $\alpha Y$ and $(1-\alpha)Y$.
The law of total covariance therefore gives
\begin{equation}
\begin{aligned}
\operatorname{Cov}(Y^{(1)},Y^{(2)}\mid y)
&=
\alpha(1-\alpha)
[\operatorname{Var}(Y\mid y)-m_Q(y)-\sigma^2] \\
&=
\alpha(1-\alpha)[\mu-m_Q(y)],
\end{aligned}
\end{equation}
which proves Eq.~\eqref{eq:app_raw_cov}.

From Eq.~\eqref{eq:branch_sinogram_correction},
$\tilde y_{\mathrm{vm}}^{(k)}-y
=a_kr_k+(1-a_k)e_{\mathrm{pc}}$.
Expanding the conditional covariance and applying the
Cauchy--Schwarz inequality gives Eq.~\eqref{eq:app_vm_bound}.
\IEEEQEDoff
\end{IEEEproof}

\bibliographystyle{IEEEtran}
\bibliography{references}

\end{document}